\documentclass{article}

\usepackage{microtype}
\usepackage{graphicx}
\usepackage{subfigure}
\usepackage{booktabs} 
\usepackage{algorithm}
\usepackage{algorithmic}
\usepackage[compress]{cite}

\usepackage{amsmath,amssymb,amsthm}
\usepackage{dsfont}
\usepackage{url}
\usepackage{tikz}
\usepackage[hidelinks]{hyperref}
\usepackage{psfrag}

\usepackage{hyperref}

\newtheorem{assumption}{Assumption}
\newtheorem{theorem}{Theorem}
\newtheorem{definition}{Definition}

\newtheorem{proposition}{Proposition}

\DeclareMathOperator*{\argmin}{arg\,min}

\DeclareMathOperator{\sign}{sign}

\begin{document}

\title{Gradient Sparsification Can Improve Performance of Differentially-Private Convex Machine Learning}

\author{Farhad Farokhi\thanks{F. Farokhi is with the Department of Electrical and Electronic Engineering, The University of Melbourne, Australia. e-mail: ffarokhi@unimelb.edu.au}}

\date{}

\maketitle

\begin{abstract}
We use gradient sparsification to reduce the adverse effect of differential privacy noise on performance of private machine learning models. To this aim, we employ compressed sensing and additive Laplace noise to evaluate differentially-private gradients.  Noisy privacy-preserving gradients are used to perform stochastic gradient descent for training machine learning models. Sparsification, achieved by setting the smallest gradient entries to zero, can reduce the convergence speed of the training algorithm. However, by sparsification and compressed sensing, the dimension of communicated gradient and the magnitude of additive noise can be reduced. The interplay between these effects determines whether gradient sparsification improves the performance of differentially-private machine learning models. We investigate this analytically in the paper. We prove that, for small privacy budgets, compression can improve performance of privacy-preserving machine learning models. However, for large privacy budgets, compression does not necessarily improve the performance. Intuitively, this is because the effect of privacy-preserving noise is minimal in large privacy budget regime and thus improvements from gradient sparsification cannot compensate for its slower convergence. 
\end{abstract}

\section{Introduction}
Recent advances in machine learning have promoted  the development of wide-ranging applications in natural language processing~\cite{wu2016google}, healthcare~\cite{jiang2017artificial}, medical research~\cite{kourou2015machine}, and cyber-security~\cite{dua2016data}. Technology giants, such as Amazon\footnote{ \url{https://aws.amazon.com/machine-learning/}}, Google\footnote{\url{https://ai.google/tools/}}, and Microsoft\footnote{\url{https://www.microsoft.com/en-us/ai/}}, offer commercial machine learning platforms. However, government regulations increasingly prohibit sharing data without consent~\cite{bennett2018revisiting}. This has motivated a surge of papers offering privacy-preserving machine learning~\cite{abadi2016deep,wu2020value,21mcsherry2009differentially, zhang2016differential,zhang2017dynamic}.

Differential privacy is a strong contender in privacy analysis and preservation~\cite{dwork2014algorithmic}. In machine learning, additive noise can be used to ensure differential privacy~\cite{abadi2016deep,wei2020federated,bonawitz2017practical}. Analysis of these algorithms often reveals that the performance degradation caused by differential privacy noise is proportional to the learning horizon (i.e., maximum number of iterations of the learning algorithm) and the number of parameters. This is because, when composing differentially-private mechanism~\cite{dwork2010boosting}, privacy budgets add up. Therefore, to keep the information leakage low, the per iteration privacy budget must be reduced inversely proportional to the learning horizon and the number of parameters. This is troubling specially in deep learning applications as the number of parameters can range in millions (depending on the number of layers and features) and the learning horizon must be large enough for the gradient information to propagate deep into the network. 

This has motivated research avenues for reducing magnitude of the privacy-preserving noise. The most studied direction is the use of moment accountant for developing less conservative compositions results in differential privacy~\cite{dwork2010boosting} and differentially-private machine learning~\cite{abadi2016deep}. These composition bounds stem from the relationship between R\'{e}nyi differential privacy~\cite{mironov2017renyi} and approximate differential privacy~\cite{dwork2006our}. The improved composition bounds improve the performance degradation by a constant factor~\cite{bassily2014private}.  Note that such an improvement can have significant impact in practice, motivating many researchers to focus on improving the composition bounds~\cite{asoodeh2020better}. Another less traveled path is gradient compression or sparsification~\cite{khirirat2018gradient}. This approach has recently\footnote{This paper was developed in parallel with/independently from~\cite{Kerkouche2020Compression}. We only became aware of the contributions of~\cite{Kerkouche2020Compression} while finalizing this paper (after independently proposing compressed sensing and differential privacy for training machine learning models). Nonetheless this is a part of research and we decided to cite that paper as they appeared online earlier. That being said, the positive empirical results in~\cite{Kerkouche2020Compression} did not motivate this paper. The motivation for this study stems from the theoretical analysis in~\cite{khirirat2018gradient}. As described in detail in Section~\ref{sec:related_work}, there are significant differences between this paper and~\cite{Kerkouche2020Compression} in addition to similarities, such as compressed sensing. } shown promises in differentially-private federated learning~\cite{Kerkouche2020Compression}. This is however observed purely experimentally/empirically. It is not known, in general, when compression or sparsification improves performance of differentially-private machine learning models. This is the topic of the current paper. 

In this paper, we use gradient sparsification (also known as compression) to improve performance of differentially-private machine learning models. We use a greedy sparsification algorithm and select $\kappa$ largest entries of the gradient at each iteration. We then use  compressed sensing and additive Laplace noise to develop noisy differentially-private gradients. Doing so, we can show that the magnitude\footnote{$\mathcal{O}(\cdot)$ denotes the Bachmann--Landau or asymptotic notation.} of additive noise is $\mathcal{O}(\kappa^2\log(n_\theta/\kappa))$, where $n_\theta$ is the number of the parameters in the machine learning model. If we did not sparsify the gradient and did not use compressed sensing, the magnitude of the noise would have been $\mathcal{O}(n_\theta)$. This is a considerable improvement if $\kappa/n_\theta \ll 1$ and $n_\theta\gg 1$. We then use the noisy gradient to perform stochastic gradient descent for training the machine learning model. Note that sparsification can reduce the convergence speed of the algorithm and therefore more iterations are required~\cite{khirirat2018gradient}. This can increase the magnitude of the privacy-preserving noise. On the other hand, by sparsification and compressed sensing, we can reduce the dimension of the communicated gradient to reduce the magnitude of the noise. The interplay between these two contracting effects determines whether gradient sparsification is useful for improving the performance of differentially-private machine learning models. 

\begin{theorem}[Informal] For large enough training horizon $T$, under appropriate assumptions (i.e., convexity, boundedness, and differentiablity), the performance of differentially-private machine learning model trained using the described stochastic gradient descent is bounded by 
\begin{align*}
\mathbb{E}\{f_{\mathcal{D}}(\theta[T])\}-f_{\mathcal{D}}(\theta^*)
\leq &
  \mathcal{O}(\sqrt{n_\theta-\kappa})
+\mathcal{O}( \kappa^2 \log(n_\theta/\kappa)\epsilon^{-1}),
\end{align*}
where $f_{\mathcal{D}}(\cdot)$ is the loss function, $\theta[T]$ is the differentially-private machine learning parameter, $\theta^*$ is the potentially privacy-intrusive, yet optimal machine learning parameter, $\epsilon$ is the differential privacy budget, and expectation is taken with respect to the differential privacy noise.
\end{theorem}

We can use Theorem~\ref{tho:1} to understand under what circumstances sparsification is useful. When $\epsilon \ll 1$, i.e., small privacy budget regime,
the performance degradation caused by privacy-preserving noise is dominated by
$\mathcal{O}(\kappa^2\log(n_\theta/\kappa) \epsilon^{-1})$. In this case, sparsification can improve the performance of privacy-preserving machine learning models. Upon selecting a small constant $\kappa$, the overall complexity is $\mathcal{O}(\log(n_\theta)\epsilon^{-1})$, which is considerably tighter than $\mathcal{O}(n_\theta\epsilon^{-1})$ in~\cite{bassily2014private} for a comparable setup. When $\epsilon\gg 1$, i.e., in large privacy budget regime, the performance degradation caused by privacy-preserving noise is dominated by $\mathcal{O}(\sqrt{n_\theta-\kappa})$, which implies that compression does not improve the performance of the privacy-preserving machine learning model. 

\section{Related Work} \label{sec:related_work}
Differential privacy is the gold standard for ensuring privacy while training machine learning models~\cite{sarwate2013signal,zhang2012functional, chaudhuri2009privacy, zhang2016differential,shokri2015privacy,abadi2016deep,mcmahan2017learning, zhang2018privacy,bassily2014private}. In order to avoid aggregating the private datasets in one location, distributed privacy-preserving machine learning has been gaining momentum~\cite{zhang2017dynamic, huang2019dp,abadi2016deep,wu2020value}. These studies, irrespective of being centralized or distributed, do not consider sparsification or compression in order to reduce the magnitude of the privacy-preserving noise. 

Various quantization and coding techniques have been used to reduce the communication cost in machine learning~\cite{alistarh2016qsgd,wen2017terngrad,wang2018atomo,konevcny2016federated,bernstein2018signsgd,lin2017deep,seide20141,khirirat2018gradient,bellet2015distributed,magnusson2017convergence}. For instance, in~\cite{bernstein2018signsgd}, the data owners only transmit the sign vector of the gradients to the learner, where these vectors are then aggregated by majority voting for each entry. 
A similar approach was considered in~\cite{magnusson2017convergence} for general optimization problems albeit in a centralized manner. Another approach is the greedy quantization~\cite{khirirat2018gradient} in which the largest entries of the gradient are used. This method ensures convergence to the exact optimizer while using the signed vector only convergence to a bounded neighborhood of the optimizer can be achieved. 
These studies do not incorporate differential privacy in their algorithms. 

A very recent study in~\cite{Kerkouche2020Compression} has considered compression and differential privacy in the context of federated learning. That paper uses Discrete Cosine Transform, see, e.g.,~\cite{ahmed1974discrete}, to compress the model updates transmitted by the data owners to the central learner. It then uses compressive sensing to further reduce the size of the machine learning model parameters. This enables them to increase the machine learning model quality without sacrificing privacy. Although \cite{Kerkouche2020Compression}  is similar to the framework presented in this paper, there are several significant differences. Firstly, we use a greedy quantization to sparsify the gradient. This slightly reduces the computational burden by avoiding  discrete cosine transform. Second, and most importantly, the performance  analysis in~\cite{Kerkouche2020Compression} is empirical and proof of convergence is not provided. Therefore, we would not able to \textit{a priori} guess whether compression helps or not. In this paper, however, we provide analytical performance bounds and investigate regimes in which sparsification improves the performance. This is of particular importance given the cost and environmental impact of training machine learning models~\cite{strubell2019energy}.

Compressed sensing have been previously used to ensure differential privacy with improved performance~\cite{li2011compressive}. Compressed sensing, also known as compressive sensing, is a technique for efficiently reconstructing vectors by finding solutions to under-determined linear systems~\cite{donoho2006compressed, candes2005decoding, candes2006robust}. The reconstruction is made possible by assuming that the underlying vector is sparse in some domain. Prior to~\cite{Kerkouche2020Compression}, compressed sensing  was not used to improve performance of differentially-private machine learning and, as statede earlier, that paper does not prove convergence and lacks explicit performance bounds for privacy-aware machine learning. 

\begin{algorithm}[t]
	\caption{\label{alg:0} CoSaMP: Compressive sensing recovery algorithm.}
	\begin{algorithmic}[1]
		\REQUIRE Matrix $\Psi$, sparsity level $\kappa$, noisy observation $y$
		\ENSURE $\kappa$-sparse vector  $\mathrm{CoSaMP}(\Psi,\kappa,y)$
		\STATE $z^0\leftarrow 0$ and $u\leftarrow y$
		\FOR{$k=1,\dots,6(\kappa+1)$}
		\STATE $\mathcal{T}\leftarrow \{\mbox{indices for }2\kappa\mbox{ largest entries of }\Psi^\top u\}\cup \{i:z^{k-1}\neq 0\}$
		\STATE $(b_i)_{i\in \mathcal{T}}\leftarrow \Psi_{\mathcal{T}}^\dag y$, where $\Psi_{\mathcal{T}}$ is a sub-matrix of $\Psi$ obtained by removing columns of $\Psi$ whose index is not in $\mathcal{T}$ and $^\dag$ denotes the Moore--Penrose inverse
		\STATE $(b_i)_{i\notin\mathcal{T}}\leftarrow 0$
		\STATE $\mathcal{T}'\leftarrow \{\mbox{indices for }\kappa\mbox{ largest entries of }b\}$
		\STATE $(z^k_i)_{i\in\mathcal{T}'}\leftarrow (b_i)_{i\in\mathcal{T}'}$
		\STATE $(z^k_i)_{i\notin\mathcal{T}'}\leftarrow 0$
		\STATE $u\leftarrow y-\Psi z^k$
		\ENDFOR
		\STATE Return $z^k$
	\end{algorithmic}
\end{algorithm}

\section{Differentially-Private Machine Learning}
Consider a private training dataset $\mathcal{D}:=\{(x_i,y_i)\}_{i=1}^{n} \subseteq\mathbb{X}\times\mathbb{Y}\subseteq \mathbb{R}^{n_x}\times \mathbb{R}^{n_y}$ for supervised machine learning, where $x_i\in\mathbb{X}$ and $y_i\in\mathbb{Y}$ denote inputs and  outputs, respectively. Note that $n_x,n_y>0$ are constant integers denoting the dimension of input and output spaces. We are interested in extracting a meaningful relationship between inputs and outputs using machine learning model $\mathfrak{M}: \mathbb{R}^{n_x} \times \mathbb{R}^{n_\theta} \rightarrow \mathbb{R}^{n_y}$ based on the training dataset $\mathcal{D}$, where $n_\theta>0$ is an integer denoting the dimension of machine learning model parameter. This is done by solving the optimization problem:
\begin{align}\label{eqn:machine learning}
\theta^*\in\argmin_{\theta\in\Theta}\Bigg[\lambda(\theta)+\frac{1}{n} \sum_{\{x,y\}\in\mathcal{D}}\ell (\mathfrak{M}(x;\theta),y)\Bigg],
\end{align}
where $\ell(\mathfrak{M}(x;\theta),y)$ is a loss function measuring the closeness of the outcome of the machine learning model $\mathfrak{M}(x;\theta)$ against the actual output $y$,  $\lambda(\theta)$ is a regularizing term, and $\Theta$ is the set of feasible machine learning model parameters. We use $f_{\mathcal{D}}(\theta)$ to denote the cost function of~\eqref{eqn:machine learning}, i.e., 
\begin{align*}
f_{\mathcal{D}}(\theta):=\lambda(\theta)+\frac{1}{n}\sum_{\{x,y\}\in\mathcal{D}} \ell(\mathfrak{M}(x;\theta),y).
\end{align*}  
The optimization-based machine learning formulation in~\eqref{eqn:machine learning} captures many relevant machine learning models, such as linear regression, support vector machine, and neural networks. In regression models, we have $y=\mathfrak{M}(x;\theta):=\theta^\top \psi(x)$, where $\theta$ is the machine learning model parameter and $\psi$ is any function. We train the regression model by solving the optimization problem~\eqref{eqn:machine learning} with $\ell(\mathfrak{M}(x;\theta),y):=\|y-\mathfrak{M}(x;\theta)\|_2^2$ and $\lambda(\theta):=0$. For linear support vector machines, the goal is to obtain a  classification rule $y=\sign(\mathfrak{M}(x;\theta))$ with $\mathfrak{M}(x;\theta):=\theta^\top [x^\top \; 1 ]^\top$ to group the training data into two classes of $y=\pm 1$.  We can train the support vector machine model by solving the optimization problem~\eqref{eqn:machine learning} with $\lambda(\theta):=(1/2)\theta^\top \theta $ and $ \ell(\mathfrak{M}(x;\theta),y):=\max(0,1-\mathfrak{M}(x;\theta)y).$ For neural networks, $\mathfrak{M}(x;\theta)$ describes the input-output behavior of the neural network with $\theta$ capturing parameters, such as internal weights and thresholds. This problem can be cast as~\eqref{eqn:machine learning} with $\lambda(\theta):=0$ and $\ell(\mathfrak{M}(x;\theta),y):=\|y-\mathfrak{M}(x;\theta))\|_2^2.$

We use differentially-private gradient descent for training the machine learning model. At iteration $k\in\{1,\dots, T\}$ with training horizon $T$, we compute the gradient $\nabla_\theta f_{\mathcal{D}}(\theta[k])$. We then greedily sparsify the gradient by selecting its largest $\kappa\in\{1,\dots,n_\theta\}$ entries to extract $Q(\nabla_\theta f_{\mathcal{D}}(\theta[k]))$, where, for any vector  $v\in\mathbb{R}^{n_\theta}$, 
\begin{align} \label{eqn:sparse}
[Q(v)]_{\pi(i)}
:=
\begin{cases}
[v]_{\pi(i)}, & i\leq \kappa,\\
0, & \mbox{otherwise},
\end{cases}
\end{align}
where $\pi$ is a permutation of $\{1,\dots,n_\theta\}$ such that $|[v]_{\pi(j)}|\geq |[v]_{\pi(j+1)}|$ for all $j\in\{1,\dots,n_\theta-1\}$, $[v]_i$ is the $i$-the entry of $v$, and $[Q(v)]_i$ is the $i$-the entry of $Q(v)$. Note that, by definition, in $Q(\nabla_\theta f_{\mathcal{D}}(\theta[k]))$, the largest $\kappa$ elements are equal to $\nabla_\theta f_{\mathcal{D}}(\theta[k])$ and the rest are set to zero. The outcome is a $\kappa$-sparse vector. Now that we have established this, we use the compressive differential privacy mechanism from~\cite{li2011compressive} to generate private gradients. This is summarized in Algorithm~\ref{alg:1}. We make the following standard assumption in differentially-private machine learning.

\begin{assumption} \label{assum:bounded_gradient}
	$|\partial \ell(\mathfrak{M}(x;\theta),y)/\partial \theta_i|\leq \Xi$ for all $\theta\in\Theta$, $(x,y)\in\mathbb{X}\times\mathbb{Y}$, and all $i$.
\end{assumption}

Assumption~\ref{assum:bounded_gradient} is without loss of generality if the set of feasible parameters $\Theta$ is compact and the loss function $\ell(\mathfrak{M}(x;\theta),y)$ is continuously differentiable with respect to the model parameters $\theta$. This is due to the Weierstrass extreme value theorem~\cite{rudin1976principles}.

\begin{definition}[Differential Privacy] \label{def:dp} A randomized algorithm $\mathcal{A}$  is  $\epsilon$-differentially private if
	\begin{align*}
	\mathbb{P}\bigg\{\mathcal{A}(\mathcal{D})\in\mathbb{S}\bigg\}
	&\leq \exp(\epsilon)\mathbb{P}\bigg\{\mathcal{A}(\mathcal{D}') \in\mathbb{S}\bigg\},
	\end{align*}
	where $\mathbb{S}$ is any measurable set and adjacent datasets $\mathcal{D}$ and $\mathcal{D}'$  differing in at most in one entry, i.e., $|\mathcal{D}\setminus\mathcal{D}'|=|\mathcal{D}'\setminus\mathcal{D}|\leq 1$.
\end{definition}

Now, we can formalize our notion of privacy in this paper by proving that  Algorithm~\ref{alg:1} guarantees differential privacy.

\begin{proposition} \label{proposition:1} The random mechanism in Algorithm~\ref{alg:1} is $\epsilon$-differentially private.
\end{proposition}

\begin{proof} Let $\mathfrak{P},\mathfrak{P}'\in\mathbb{R}^{n_\theta\times n_\theta}$ be selection matrices such that $Q(\nabla_\theta f_{\mathcal{D}}(\theta[k]))=\mathfrak{P} \nabla_\theta f_{\mathcal{D}}(\theta[k])$ and $Q(\nabla_\theta f_{\mathcal{D}'}(\theta[k]))=\mathfrak{P}' \nabla_\theta f_{\mathcal{D}'}(\theta[k])$. 
Note that 
\begin{align*}
\|(Q(\nabla_\theta f_{\mathcal{D}}(\theta[k]))- Q(\nabla_\theta f_{\mathcal{D}'}(\theta[k])))/(2\Xi \kappa/n)\|_1
\leq &\frac{1}{2\Xi \kappa}\|\mathfrak{P} \nabla_\theta f_{\mathcal{D}}(\theta[k])- \mathfrak{P}' \nabla_\theta f_{\mathcal{D}'}(\theta[k])\|_1\\
\leq &\frac{1}{2\Xi \kappa}(\|\mathfrak{P} \nabla_\theta f_{\mathcal{D}}(\theta[k])\|+\| \mathfrak{P}' \nabla_\theta f_{\mathcal{D}'}(\theta[k])\|_1)\\
\leq &1.
\end{align*}
The rest follows from Lemma 3 in~\cite{li2011compressive}. 
\end{proof}

Let us investigate the utility of differentially-private vectors generated by Algorithm~\ref{alg:1}. The magnitude of the noise in this step plays an important role in the performance of differentially-private machine learning models.

\begin{proposition} \label{prop:bound_error_v}
There exists $C>0$ such that 
\begin{align*}
\mathbb{E}\{\|\mathcal{M}(v)-v\|_2^2\}\leq \frac{C\Xi^2 \kappa^4 \log^2(n_\theta/\kappa)}{\epsilon^2 }.
\end{align*}
\end{proposition}

\begin{algorithm}[t]
	\caption{\label{alg:1} Compressive differential privacy mechanism}
	\begin{algorithmic}[1]
		\REQUIRE Privacy budget $\epsilon$ and $\kappa$-sparse vector $v$
		\ENSURE Differentially-private output $\mathcal{M}(v)$
		\STATE $p\leftarrow \mathcal{O}(\kappa\log(n_\theta/\kappa))$ such that $p\geq 3$
		\STATE Generate a random sampling matrix $\Psi\in\mathbb{R}^{p\times n_\theta}$ such that entries of $\sqrt{p}\Psi$ are zero-mean Gaussian variables with unit variance 
		\STATE $y\leftarrow \Psi (v/(2\Xi \kappa))+e$, where $e$ is a vector i.i.d. zero-mean Laplace noise with scale $\sqrt{p}/\epsilon$
		\STATE $v^*\leftarrow \mathrm{CoSaMP}(\Psi,\kappa,y)$ 
		\STATE Return $\mathcal{M}(v)=(2\Xi \kappa)v^*$
	\end{algorithmic}
\end{algorithm}

\begin{proof}
	Recall that $e$ is a vector of i.i.d. zero-mean Laplace noises with scale $b:=\sqrt{p}/\epsilon$.
	Note that $\{e\,|\,\|e\|_2\geq \alpha\}\subseteq \{e\,|\,\exists i:|e_i|\geq \alpha/\sqrt{p}\}$. Therefore, we have 
		\begin{align*}
		\mathbb{P}\{\|e\|_2\geq \alpha\} 
		&\leq  \mathbb{P}\{e\,|\,\exists i:|e_i|\geq \alpha/\sqrt{p}\}\\
		&= p \mathbb{P}\{|e_i|\geq \alpha/\sqrt{p}\}\\
		&=p\exp(-\alpha\epsilon /p)\\
		&=\exp(-\alpha \epsilon \log(p)/p).
		\end{align*}
	From~\cite{needell2009cosamp}, we know that there exists $C'>0$ such that 
	$\|v^*-v/(2\Xi \kappa/n)\|_2\leq C' \alpha$ if $\|e\|_2\leq \alpha$. Therefore,
	\begin{align*}
	\mathbb{P}\{\|v^*-v/(2\Xi \kappa)\|_2\geq C' \alpha\}\leq \mathbb{P}\{\|e\|_2\geq \alpha\}.
	\end{align*}
	Now, note that
	\begin{align*}
\mathbb{E}\left\{\left\|v^*\hspace{-.03in}-\hspace{-.03in}\frac{1}{2\Xi \kappa} v\right\|_2^2\right\}
\hspace{-.03in}=\hspace{-.03in} &\int_0^\infty \mathbb{P}\left\{\left\|v^*\hspace{-.03in}-\hspace{-.03in}\frac{1}{2\Xi \kappa} v\right\|_2^2\geq t\right\} \mathrm{d}t\\
\hspace{-.03in}= \hspace{-.03in}&\int_0^\infty \mathbb{P}\left\{\left\|v^*\hspace{-.03in}-\hspace{-.03in}\frac{1}{2\Xi \kappa} v\right\|_2\geq \sqrt{t}\right\} \mathrm{d}t\\
\leq &\int_0^\infty \mathbb{P}\{\|e\|_2\geq \sqrt{t}/C'\} \mathrm{d}t\\
\leq &\int_0^\infty \exp\left(-\frac{\sqrt{t}\epsilon \log(p)}{C' p}\right) \mathrm{d}t\\
=&\frac{2C^{\prime 2} p^2}{\epsilon^2 \log(p)^2}\\
\leq & \frac{2C^{\prime 2} p^2}{\epsilon^2}\\
\leq & \frac{2C^{\prime 2}C^{\prime\prime 2} \kappa^2 \log^2(n_\theta/\kappa)}{\epsilon^2},
\end{align*}
where  the penultimate inequality follows from that $p\geq 3$ and the last inequality is because $p\leq C'' \kappa \log(n_\theta/\kappa)$.
As a result, 
\begin{align*}
\mathbb{E}\{\|\mathcal{M}(v)-v\|_2^2\}
&=\left({2\Xi \kappa}\right)^2
\mathbb{E}\{\|v^*-v/(2\Xi \kappa/n)\|_2^2\}\\
&\leq \frac{8C^{\prime 2}C^{\prime\prime 2}\Xi^2 \kappa^4 \log^2(n_\theta/\kappa)}{\epsilon^2}.
\end{align*}
The rest follows from defining $C=8C^{\prime 2}C^{\prime\prime 2}$.
\end{proof}

Now, we use the differentially-private sparsified gradient $\mathcal{M}(Q(\nabla_\theta f(\theta[k])))$ extracted from Algorithm~\ref{alg:1} to train the machine learning model using the stochastic gradient descent algorithm:
\begin{align}
\theta[k+1]=\theta[k]-\rho_k \mathcal{M}(Q(\nabla_\theta f(\theta[k]))),
\end{align}
where $\rho_k>0$ refers to the step size (in the optimization literature) or the learning rate (in the machine learning literature). Note that, if we define the additive privacy noise $w[k]=\mathcal{M}(Q(\nabla_\theta f(\theta[k])))-Q(\nabla_\theta f(\theta[k]))$, we get
\begin{align}
\theta[k+1]=\theta[k]-\rho_k(Q(\nabla_\theta f(\theta[k]))+w[k]).
\end{align}
Now, we can use the same line of reasoning as in~\cite{khirirat2018gradient} and~\cite{shamir2013stochastic} to develop a performance bound for the privacy-preserving machine learning model. Before presenting the proof, we make the following standing assumption.

\begin{assumption} \label{assum:iid}
$w[k]$ and $w[k']$ are statistically independent for all $k\neq k'$.
\end{assumption}

Note that noise $w[k]$ captures the effect of the additive noise in Step 3 of Algorithm~\ref{alg:1}. In fact, $w[k]=0$ if there is no additive noise. Noting that the additive noises are independently and identically distributed, this assumption feels intuitively true but proving it is a challenge task that is left for future research.

\begin{algorithm}[t]
	\caption{\label{alg:2} Sparsifying stochastic gradient algorithm with differential privacy and compressive sensing.}
	\begin{algorithmic}[1]
		\REQUIRE Privacy budget $\epsilon$, initialization $\theta[1]$
		\ENSURE Privacy-preserving machine learning model $\theta[T]$
		\FOR{$k=1,\dots,T-1$}
		\STATE $g[k]\leftarrow \nabla_\theta f(\theta[k])$
		\STATE $g'[k]\leftarrow Q(g[k])$
		\STATE $\tilde{g}[k]\leftarrow \mathcal{M}(g'[k])$ by Algorithm~\ref{alg:1}
		\STATE $\theta[k+1]\leftarrow \theta[k]-\rho_k\tilde{g}[k]$
		\ENDFOR
		\STATE Return $\theta[T]$
	\end{algorithmic}
\end{algorithm}

\setcounter{theorem}{0}
\begin{theorem} \label{tho:1}
	In addition to Assumptions~\ref{assum:bounded_gradient} and~\ref{assum:iid}, suppose that $f_{\mathcal{D}}$ is convex and continuously differentiable, $\Theta$ is convex, there exist a constant $D$ such that $\mathrm{diam}(\Theta):=\sup_{\theta,\theta'} \|\theta-\theta'\|_2\leq D$, and the learning rate $\rho_k=c/\sqrt{k}$ for some constant\footnote{Following the proof of the theorem, we can see that $c=D/(2G)$, where $C$ is given in the proof of Proposition~\ref{prop:bound_error_v} and $G=\Xi\kappa^{1/2}(1+{C\kappa^3 \log^2(n_\theta/\kappa)}/{(\epsilon^2}))^{1/2}$.} $c>0$. Then, the iterates of Algorithm~\ref{alg:2} satisfy
\begin{align*}
\mathbb{E}\{f_{\mathcal{D}}(\theta[T])\}-f_{\mathcal{D}}(\theta^*)
\leq D\Xi(2+\log(T))
\Bigg(& 
\frac{\sqrt{2}}{\sqrt{T}}\sqrt{\kappa\left(1+\frac{C\kappa^3 \log^2(n_\theta/\kappa)}{\epsilon^2 }\right)}+\sqrt{2(n_\theta-\kappa)}+\frac{\sqrt{C}\kappa^2\log(n_\theta/\kappa)}{\epsilon} \Bigg).
\end{align*}
\end{theorem}

\begin{proof}
	The proof follows a similar line of reasoning as in~\cite[Thereom 2]{shamir2013stochastic}, albeit with extra complications due to sparsification of the gradient. 
Note that $\mathbb{E}\{\|Q(\nabla_\theta f_{\mathcal{D}}(\theta[k]))+w[k]\|_2^2\}\leq G^2$, where
\begin{align*}
G^2:= \Xi^2\kappa\left(1+\frac{C\kappa^3 \log^2(n_\theta/\kappa)}{\epsilon^2}\right).
\end{align*}
By convexity of $\Theta$, for any $\theta\in\Theta$, we get
\begin{align*}
\mathbb{E}\{\|\theta[k+1]-\theta\|_2^2\}
=&\mathbb{E}\{\|\Pi_\Theta(\theta[k]-\rho_k (Q(\nabla_\theta f_{\mathcal{D}}(\theta[k]))+w[k]))-\theta\|\}\\
\leq &\mathbb{E}\{\|\theta[k]-\rho_k (Q(\nabla_\theta f_{\mathcal{D}}(\theta[k]))+w[k])-\theta\|\}\\
\leq & \mathbb{E}\{\|\theta[k]-\theta\|_2^2\}+\rho_k^2 G^2-2\rho_k \mathbb{E}\{\langle Q(\nabla_\theta f_{\mathcal{D}}(\theta[k]))+w[k],\theta[k]-\theta\rangle\}.
\end{align*}
Therefore, 
\begin{align*}
\mathbb{E}\{\langle Q(\nabla_\theta f_{\mathcal{D}}(\theta[k]))+w[k],\theta[k]-\theta\rangle\}
\leq &\frac{\mathbb{E}\{\|\theta[k]-\theta\|_2^2\}-  \mathbb{E}\{\|\theta[k+1]-\theta\|_2^2\}}{2\rho_k}+\frac{G^2}{2}\rho_k.
\end{align*}
As a result, for any $t\in\{1,\dots,T \}$, we get
\begin{align}
\sum_{k=T-t}^T \mathbb{E}\{\langle Q(\nabla_\theta f_{\mathcal{D}}(\theta[k]))+w[k],\theta[k]-\theta\rangle\}\nonumber
\leq &\frac{1}{2\rho_{T-t}} \mathbb{E}\{\|\theta[T-t]-\theta\|_2^2\}
+\frac{G^2}{2}\sum_{k=T-t}^{T}\rho_k\nonumber\\
&+\sum_{k=T-t+1}^T \frac{\mathbb{E}\{\|\theta[k]-\theta\|_2^2\}}{2}\left(\frac{1}{\rho_k}-\frac{1}{\rho_{k+1}}\right).\label{eqn:4}
\end{align}
Setting $\theta=\theta[T-t]$, $\rho_k=c/\sqrt{k}$, and $\|\theta[k]-\theta\|_2^2\leq D^2$ in~\eqref{eqn:4}, we get
\begin{align}
\sum_{k=T-t}^T\mathbb{E}\{\langle Q(\nabla_\theta f_{\mathcal{D}}(\theta[k]))+w[k],\theta[k]-\theta[T-t]\rangle\}
\leq \frac{D^2}{2c}(\sqrt{T}-\sqrt{T-t})
+\frac{G^2 c}{2}\sum_{k=T-t}^{T}\frac{1}{\sqrt{k}}.\label{eqn:1}
\end{align}
Convexity and differentiablity of $f$ implies that
\begin{align}
\langle\nabla_\theta f_{\mathcal{D}}(\theta[k]),\theta[k]-\theta[T-t]\rangle
&\geq f_{\mathcal{D}}(\theta[k])-f_{\mathcal{D}}(\theta[T-t]).
\label{eqn:convex_f_lower_bound}
\end{align}
Further, 
\begin{align}
\mathbb{E}\{\langle w[k],\theta[k]-\theta[T-t]\rangle\}
&\geq \mathbb{E}\{-|\langle w[k],\theta[k]-\theta[T-t]\rangle|\}\nonumber\\
&\geq \mathbb{E}\{-\|w[k]\|_2\|\theta[k]-\theta[T-t]\|_2\}\nonumber\\
&\geq -WD, \label{eqn:proof:lowerbound_w_theta}
\end{align}
where the first inequality follows from the definition of absolute value, the second inequality follows from the Cauchy--Schwarz inequality, and the third inequality follows from $\mathbb{E}\{\|w[k]\|_2^2\}\leq W^2$ and $\|\theta[k]-\theta[T-t]\|_2^2\leq D^2$. It is worth mentioning that combining Proposition~\ref{prop:bound_error_v} and the Jensen's inequality~\cite[Proposition 3.5.1]{lange2010applied} shows that
\begin{align*}
W:=\frac{\sqrt{C}\Xi \kappa^2\log(n_\theta/\kappa)}{\epsilon}.
\end{align*}
Now, note that 
\begin{align*}
\mathbb{E}\{\langle Q(\nabla_\theta f_{\mathcal{D}}(\theta[k]))+w[k],\theta[k]-\theta[T-t]\rangle\}
=&\mathbb{E}\{\langle Q(\nabla_\theta f_{\mathcal{D}}(\theta[k]))-\nabla_\theta f_{\mathcal{D}}(\theta[k]),\theta[k]-\theta[T-t]\rangle\}\\
&+\mathbb{E}\{\langle\nabla_\theta f_{\mathcal{D}}(\theta[k]),\theta[k]-\theta[T-t]\rangle\}\\
& +\mathbb{E}\{\langle w[k],\theta[k]-\theta[T-t]\rangle\} \\
\geq& \mathbb{E}\{\langle Q(\nabla_\theta f_{\mathcal{D}}(\theta[k]))-\nabla_\theta f_{\mathcal{D}}(\theta[k]),\theta[k]-\theta[T-t]\rangle\}\\
&+\mathbb{E}\{f_{\mathcal{D}}(\theta[k])-f_{\mathcal{D}}(\theta[T-t])\}-WD
\end{align*}
where the inequality follows~\eqref{eqn:convex_f_lower_bound} and~\eqref{eqn:proof:lowerbound_w_theta}. For any vector $v\in\mathbb{R}^{n_\theta}$, we have
\begin{align*}
\|Q(v)-v\|_2^2
=&\|Q(v)\|_2^2+\|v\|_2^2-2\langle Q(v),v\rangle\\
\leq & 2(1-\kappa/n_\theta) \|v\|_2^2,
\end{align*}
where the inequality follows from Lemma~1 in~\cite{khirirat2018gradient}. By using 
the Cauchy--Schwarz inequality and Assumption~\ref{assum:bounded_gradient}, we get
\begin{align*}
\mathbb{E}\{\langle Q(\nabla_\theta f_{\mathcal{D}}(\theta[k]))-\nabla_\theta f_{\mathcal{D}}(\theta[k]),\theta[k]-\theta[T-t]\rangle\}
&\geq -\|\theta[k]-\theta[T-t]\|_2 \|Q(\nabla_\theta f_{\mathcal{D}}(\theta[k]))-\nabla_\theta f_{\mathcal{D}}(\theta[k])\|_2\\
&\geq -D\sqrt{2(1-\kappa/n_\theta)} \|\nabla_\theta f_{\mathcal{D}}(\theta[k])\|_2\\
&\geq -D\sqrt{2(n_\theta-\kappa)} \Xi.
\end{align*}
Hence, 
\begin{align}
\mathbb{E}\{\langle Q(\nabla_\theta f_{\mathcal{D}}(\theta[k]))+w[k],\theta[k]-\theta[T-t]\rangle\}
\geq& \mathbb{E}\{f_{\mathcal{D}}(\theta[k])-f_{\mathcal{D}}(\theta[T-t])\}-D(\sqrt{2(n_\theta-\kappa)} \Xi+W).\label{eqn:2}
\end{align}
Substituting~\eqref{eqn:2} into~\eqref{eqn:1} while noting that $\sum_{k=T-t}^T 1/\sqrt{k}\leq 2(\sqrt{T}-\sqrt{T-t-1})$, we get
\begin{align}
\sum_{k=T-t}^T \mathbb{E}\{f_{\mathcal{D}}(\theta[k])-f_{\mathcal{D}}(\theta[T-t])\}
\leq &\left(\frac{D^2}{2c}+G^2 c\right)(\sqrt{T}-\sqrt{T-t-1})+D(\sqrt{2(n_\theta-\kappa)} \Xi+W) t\nonumber\\
\leq &\left(\frac{D^2}{2c}+G^2 c\right)\frac{t+1}{\sqrt{T}}
+D(\sqrt{2(n_\theta-\kappa)} \Xi+W)t.\label{eqn:3}
\end{align}
Define 
\begin{align*}
S_t:=\frac{1}{t+1} \sum_{T-t}^T f_{\mathcal{D}}(\theta[k]).
\end{align*}
The inequality in~\eqref{eqn:3} gives
\begin{align*}
-\mathbb{E}\{f_{\mathcal{D}}(\theta[T-t])\}
\leq &-\mathbb{E}\{S_t\}+\frac{D^2/(2c)+cG^2}{\sqrt{T}}+D(\sqrt{2(n_\theta-\kappa)} \Xi+W) \frac{t}{t+1}\\
\leq &-\mathbb{E}\{S_t\}+\frac{D^2/(2c)+cG^2}{\sqrt{T}}+D(\sqrt{2(n_\theta-\kappa)} \Xi+W).
\end{align*}
Now, note that
\begin{align*}
t\mathbb{E}\{S_{t-1}\}
=&(t+1)\mathbb{E}\{S_t\}-\mathbb{E}\{f_{\mathcal{D}}(\theta[T-t])\}\\
\leq &t\mathbb{E}\{S_t\}+\frac{D^2/(2c)+cG^2}{\sqrt{T}} +D(\sqrt{2(n_\theta-\kappa)} \Xi+W),
\end{align*}
and, as a result,
\begin{align*}
\mathbb{E}\{S_{t-1}\}
\leq \mathbb{E}\{S_t\}&+\frac{D^2/(2c)+cG^2}{t\sqrt{T}} +\frac{D(\sqrt{2(n_\theta-\kappa)} \Xi+W)}{t}.
\end{align*}
This gives
\begin{align*}
\mathbb{E}\{f_{\mathcal{D}}(\theta[T])\}=&\mathbb{E}\{S_0\}\\
\leq  &\mathbb{E}\{S_{T-1}\} +\frac{D^2/(2c)+cG^2}{\sqrt{T}}\sum_{t=1}^{T-1}\frac{1}{t}+D(\sqrt{2(n_\theta-\kappa)} \Xi+W)\sum_{t=1}^{T-1} \frac{1}{t}.
\end{align*}
Noting that $\sum_{t=1}^{T-1}\frac{1}{t}\leq 1+\log(T)$, we get
\begin{align}
\mathbb{E}\{f_{\mathcal{D}}(\theta[T])\}
\leq & \mathbb{E}\{S_{T-1}\}+\frac{(D^2/(2c)+cG^2)(1+\log(T))}{\sqrt{T}}+D(\sqrt{2(n_\theta-\kappa)} \Xi+W)(1+\log(T)).\label{eqn:6}
\end{align}
Plugging $\theta=\theta^*$, $t=T-1$, $\rho_k=c/\sqrt{k}$, and $\|\theta[k]-\theta\|_2^2\leq D^2$  into~\eqref{eqn:4}, we get
\begin{align}
\sum_{k=1}^T \mathbb{E}\{\langle Q(\nabla_\theta f_{\mathcal{D}}(\theta[k]))+w[k],\theta[k]-\theta^*\rangle\}
\leq &\frac{D^2}{2c}(\sqrt{T}-\sqrt{T-t})
+\frac{G^2 c}{2}\sum_{k=1}^{T}\frac{1}{\sqrt{k}}.
\end{align}
Following similar steps, we can prove that
\begin{align*}
\sum_{k=1}^T \mathbb{E}\{f_{\mathcal{D}}(\theta[k])-f_{\mathcal{D}}(\theta^*)\}
\leq &\left(\frac{D^2}{2c}+G^2 c\right)\sqrt{T}+D(\sqrt{2(n_\theta-\kappa)} \Xi+W) (T-1),
\end{align*}
which shows that
\begin{align}
\mathbb{E}\{S_{T-1}\}-f_{\mathcal{D}}(\theta^*)
=&\frac{1}{T} \mathbb{E}\left\{\sum_{k=1}^T f_{\mathcal{D}}(\theta[k])-f_{\mathcal{D}}(\theta^*)\right\}\nonumber\\
\leq & \left(\frac{D^2}{2c}+G^2 c\right)\frac{1}{\sqrt{T}}+D(\sqrt{2(n_\theta-\kappa)} \Xi+W).\label{eqn:5}
\end{align}
Combining~\eqref{eqn:6} and~\eqref{eqn:5} gives
\begin{align*}
\mathbb{E}\{f_{\mathcal{D}}(\theta[T])\}-f_{\mathcal{D}}(\theta^*)
\leq&\frac{(D^2/(2c)+cG^2)(2+\log(T))}{\sqrt{T}}+D(\sqrt{2(n_\theta-\kappa)} \Xi+W)(2+\log(T)).
\end{align*}
Selecting $c=D/(\sqrt{2}G)$ simplifies this bound to
\begin{align*}
\mathbb{E}\{f_{\mathcal{D}}(\theta[T])\}-f_{\mathcal{D}}(\theta^*)
\leq&\frac{\sqrt{2}DG(2+\log(T))}{\sqrt{T}}+D(\sqrt{2(n_\theta-\kappa)} \Xi+W)(2+\log(T))\\
=&D(2+\log(T)) \left(\frac{\sqrt{2}G}{\sqrt{T}}+\sqrt{2(n_\theta-\kappa)} \Xi+W \right).
\end{align*}
This concludes the proof.
\end{proof}

\begin{figure}[t]
	\centering
	\begin{tikzpicture}
	\node[] at (8.6,0) {\includegraphics[width=0.25\linewidth]{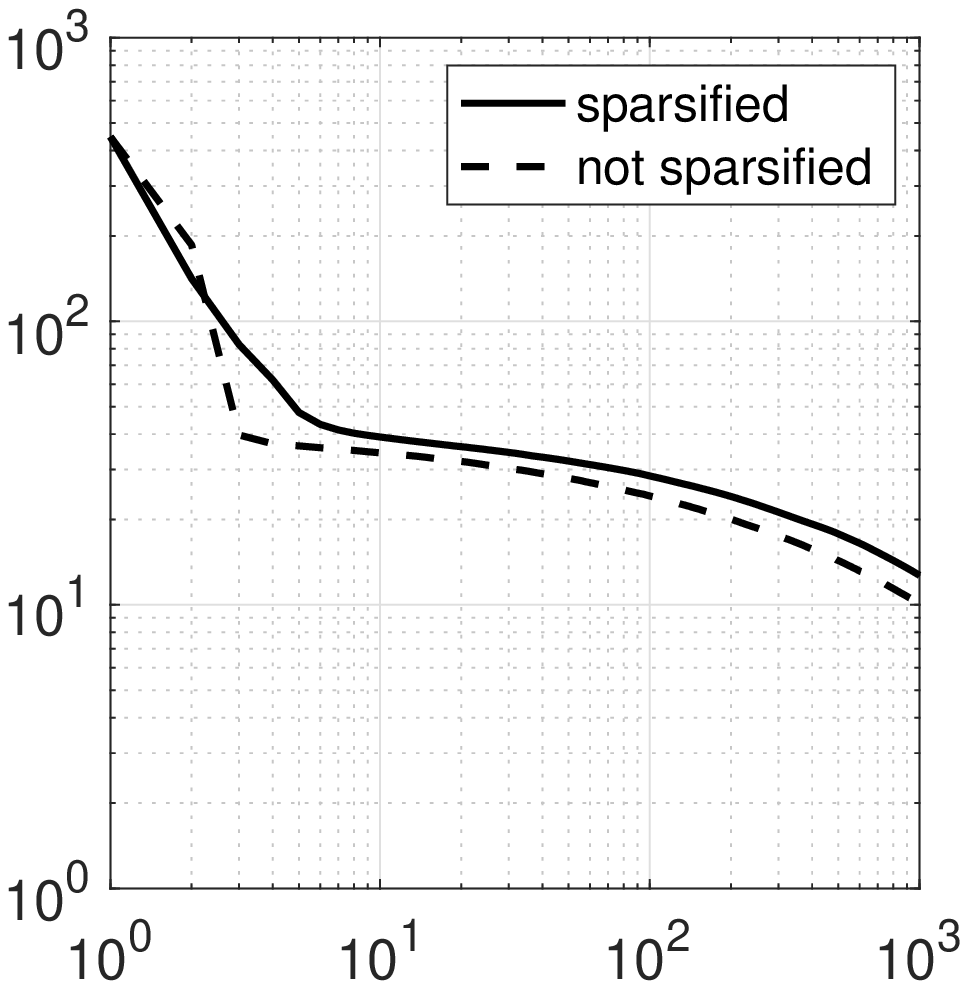}};
	\node[] at (12.9,0) {\includegraphics[width=0.25\linewidth]{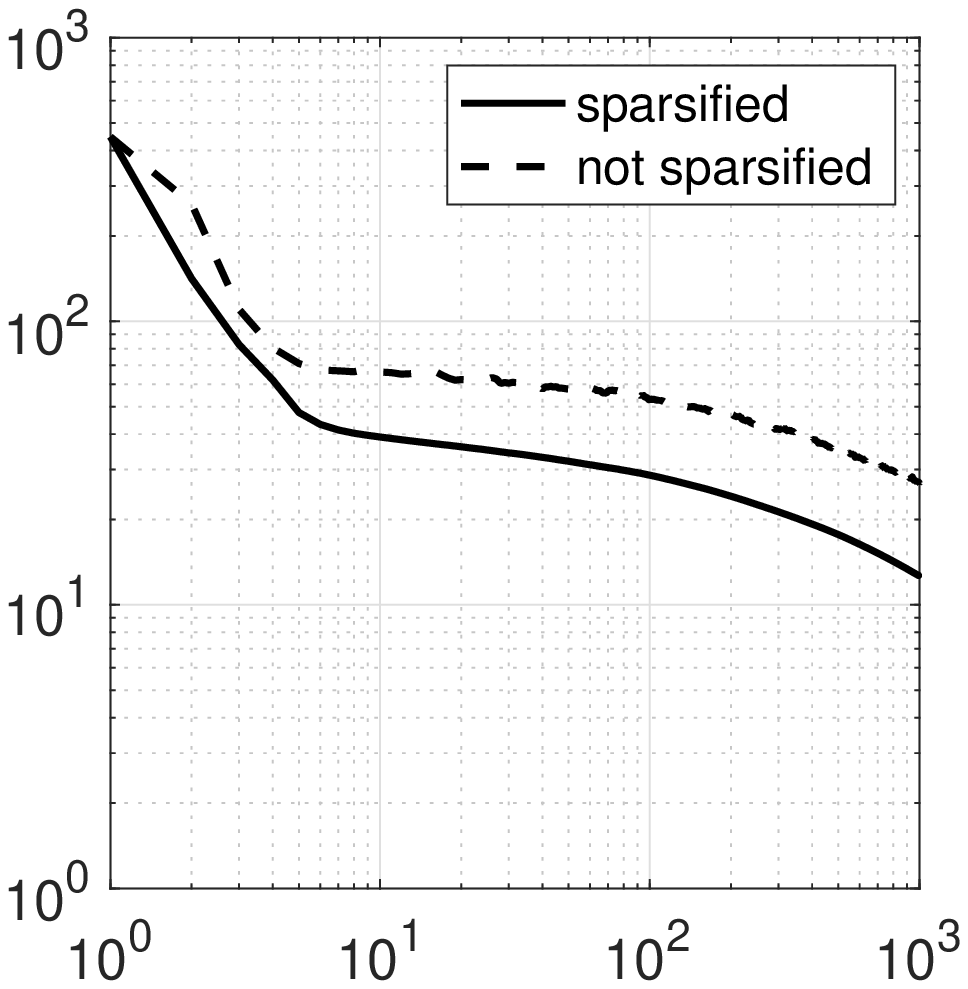}};
	\node[rotate=90] at (+6.4,0) {\small $f_{\mathcal{D}}(\theta[k])/f_{\mathcal{D}}(\theta^*)-1$};
	\node[rotate=90] at (+10.7,0) {\small $f_{\mathcal{D}}(\theta[k])/f_{\mathcal{D}}(\theta^*)-1$};
	\node[] at (8.6,-2.25) {\small iteration number $k$};
	\node[] at (12.9,-2.25) {\small iteration number $k$};
	\node[] at (8.6,2.2) {\small $\epsilon=10,n=10,000$};
	\node[] at (12.9,2.2) {\small $\epsilon=1,n=10,000$};
	\end{tikzpicture}
	\caption{
		\label{fig:simulation}
		Relative loss $f_{\mathcal{D}}(\theta[k])/f_{\mathcal{D}}(\theta^*)-1$ versus the iteration number with gradient sparsification and without gradient sparsification for two choices of privacy budget $\epsilon$ with a dataset of size $n=10,000$.
	}
\end{figure}

\begin{figure}[t]
	\centering
	\begin{tikzpicture}
	\node[] at (0,0) {\includegraphics[width=.45\linewidth]{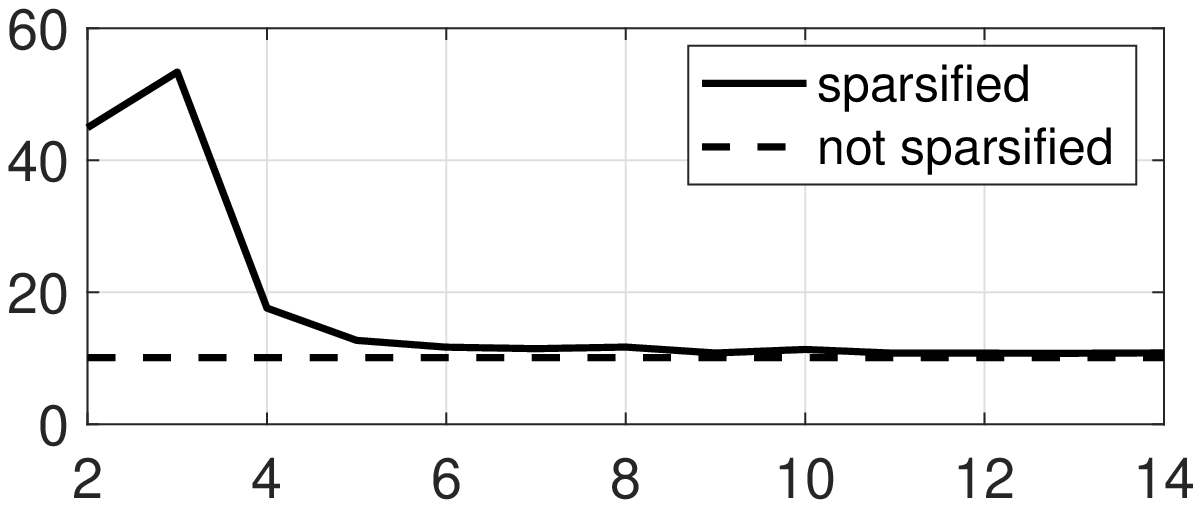}};
	\node[] at (0,-4) {\includegraphics[width=.45\linewidth]{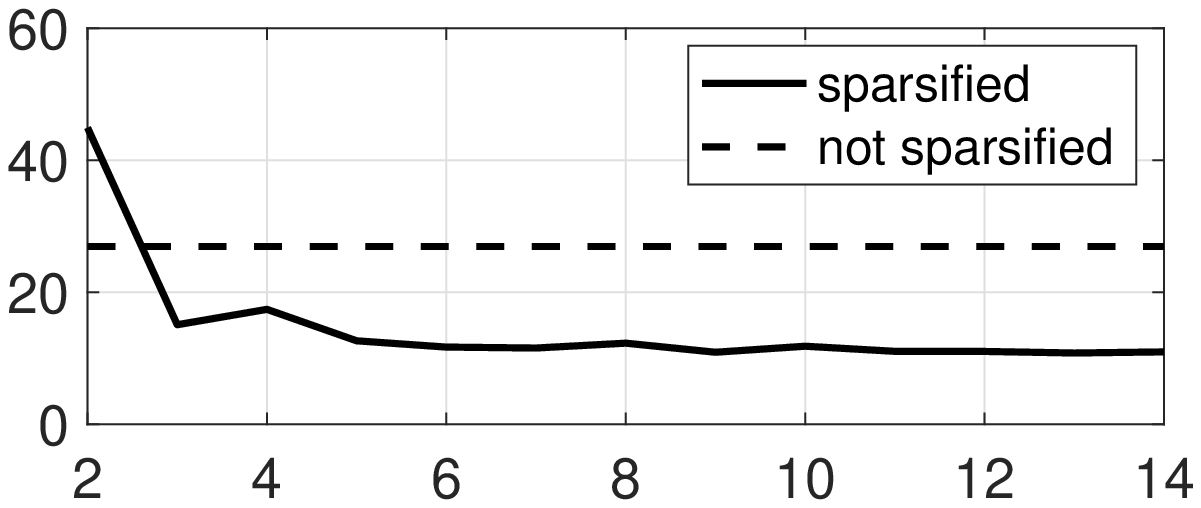}};
	\node[rotate=90] at (-3.8,0) {\small $f_{\mathcal{D}}(\theta[k])/f_{\mathcal{D}}(\theta^*)-1$};
	\node[rotate=90] at (-3.8,-4) {\small $f_{\mathcal{D}}(\theta[k])/f_{\mathcal{D}}(\theta^*)-1$};
	\node[] at (0,-1.8) {\small sparsification level $\kappa$};
	\node[] at (0,-5.8) {\small sparsification level $\kappa$};
	\node[] at (0,1.7) {\small $\epsilon=10,n=10,000$};
	\node[] at (0,-2.3) {\small $\epsilon=1,n=10,000$};
	\end{tikzpicture}
	\caption{
		\label{fig:versus}
		Relative loss $f_{\mathcal{D}}(\theta[k])/f_{\mathcal{D}}(\theta^*)-1$ versus the level of sparsification $\kappa$ with gradient sparsification and without gradient sparsification for two choices of privacy budget $\epsilon$ with a dataset of size $n=10,000$.
	}
\end{figure}
 
If the optimization or learning horizon $T$ is selected as a large enough constant, which is required to train the machine learning model sensibly, we get
\begin{align*}
\mathbb{E}\{f_{\mathcal{D}}(\theta[T])\}-f_{\mathcal{D}}(\theta^*)
 \leq &\mathcal{O}(\log(T)
(\sqrt{2(n_\theta-\kappa)}+\sqrt{C}\kappa^2\log(n_\theta/\kappa)\epsilon^{-1} )).
\end{align*}
When $\epsilon \ll 1$, i.e., small privacy budget regime, the performance degradation caused by the differential privacy noise is bounded by 
$\mathbb{E}\{f_{\mathcal{D}} (\theta[T])\}-f_{\mathcal{D}}(\theta^*)
=\mathcal{O}(\kappa^{2}\log(n_\theta/\kappa)\epsilon^{-1})$,  which implies that compression improves the performance of the privacy-preserving machine learning model. On the other hand if $\epsilon \gg 1$, i.e., high privacy budget regime, $\mathbb{E}\{f_{\mathcal{D}} (\theta[T])\}-f_{\mathcal{D}}(\theta^*)
=\mathcal{O}(\sqrt{n_\theta-\kappa})$, which implies that compression does not improve the performance of the privacy-preserving machine learning model. A more thorough discussion is presented at the end of the introduction after the informal statement of the theorem.

\section{Experimental Results}
In what follows, we demonstrate the results of the paper on a financial dataset. We use the Lending Club dataset with detailed information of around 890,000 successful loan applications on a peer-to-peer lending platform\footnote{\url{https://www.kaggle.com/wendykan/lending-club-loan-data}}. The dataset contains loan attributes, such as loan amount, and borrower information, such as credit rating and age, and interest rates of approved loans. For this experiment, we remove unique identifiers, such as id, and irrelevant attributes, such as web address for the loan listings. We further covert the categorical attributes, such as state of residence, to integers. We employ a linear regression model $y=\mathfrak{M}(x;\theta):=\theta^\top x$ with $\ell(\mathfrak{M}(x;\theta),y):= (y-\mathfrak{M}(x;\theta))^2$ and $\lambda(\theta):=0$ to infer the interest rates of the loans based on the available information.

We use the relative loss of the iterates in Algorithm~\ref{alg:2}, i.e., $f_{\mathcal{D}}(\theta[k])/f_{\mathcal{D}}(\theta^*)-1$, to illustrate the performance of differentially-private machine learning models in our experiment. The relative loss scales  the performance of the model parameter $\theta[k]$ with the performance of the optimal machine learning model $\theta^*$ without privacy constraints, effectively measuring the degradation caused by the privacy-preserving noise. We investigate the relative loss instead of the absolute loss $f_{\mathcal{D}}(\theta[k])$ or the performance degradation $f_{\mathcal{D}}(\theta[k])-f_{\mathcal{D}}(\theta^*)$ to remove the effect of exogenous factors, such as the size of the training dataset, on the performance of the optimal machine learning model. Note that relative loss is greater than or equal to zero and smaller values implies better performance. In our experiments, we have used $\rho_k=10^{-1}/\sqrt{k}$ and $T=1,000$ in Algorithm~\ref{alg:2}.

Figure~\ref{fig:simulation} illustrates the relative loss $f_{\mathcal{D}}(\theta[k])/f_{\mathcal{D}}(\theta^*)-1$ versus the iteration number with gradient sparsification and without gradient sparsification for two choices of privacy budget $\epsilon$ with a dataset of size $n=10,000$. For gradient sparsification, we used $\kappa=5$ and $p=20$.  As proved earlier, gradient sparsification improves the performance of privacy-preserving machine learning models in tight privacy budget regime, i.e., $\epsilon=1$ in our numerical example. This is no longer true for the larger privacy budget.

Figure~\ref{fig:versus} shows the relative loss $f_{\mathcal{D}}(\theta[k])/f_{\mathcal{D}}(\theta^*)-1$ versus the level of sparsification $\kappa$ with gradient sparsification and without gradient sparsification for various choices of privacy budget $\epsilon$ and size of the training dataset $n$. Again, as expected from the theoretical results, in large privacy budget regime, sparsification does not improve the performance while, in small privacy budget regime, sparsification can result in significant improvements.   

\section{Conclusions}
We used gradient sparsification to reduce the adverse effect of differential privacy noise on performance of private machine learning models. We proved that, for small privacy budgets, compression can improve performance of privacy-preserving machine learning models. However, for large privacy budgets, compression does not necessarily improve the performance. Future work can focus on extending these results to non-convex or non-smooth loss functions. 

\section*{Acknowledgments}
I would like to thank Borja Balle Pigem for pointing our a mistake in the proof of Proposition~\ref{proposition:1}.

\bibliography{example_paper}
\bibliographystyle{ieeetr}

\end{document}